\documentclass{article}
\usepackage{microtype}
\usepackage{graphicx}
\usepackage{subfigure}
\usepackage{booktabs} %

\usepackage[hyperfootnotes=false]{hyperref}
\usepackage{arxiv}

\newcommand{\myseriesplotfigwidth}{4.0cm}

\newcommand{\D}{D}

\usepackage[utf8]{inputenc}

\usepackage{tikz,ifthen,xstring,calc,pgfkeys,pgfopts}
\usepackage{tikz-uml}
\tikzumlset{fill class  = white} %

\usepackage{courier}

\usepackage{multirow}
\usepackage{array}
\usepackage{rotating}

\usepackage{listings}

\usepackage{mathtools}
\usepackage{amsmath,amssymb,amsthm}
\usepackage{thmtools}
\usepackage{xfrac}

\usepackage{graphicx}
\usepackage{adjustbox}

\usepackage{todonotes}

\usepackage[charter]{mathdesign}
\usepackage[mathcal]{eucal}
\usepackage{bbm}

\usepackage{color, colortbl}
\definecolor{Gray}{gray}{0.9}
\definecolor{LightCyan}{rgb}{0.88,1,1}

\usepackage[shortlabels]{enumitem}
\usepackage{scrextend}
\addtokomafont{labelinglabel}{\sffamily}

\usepackage{pdflscape}

\usepackage{makecell}

\usepackage{titlesec}
\titleformat{\section}
	{\normalfont\Large\bfseries\filcenter}{\thesection.}{1 ex}{}
\titleformat{\subsection}%
	{\normalfont\normalsize\bfseries}{\thesubsection.}{1 ex}{}
\titleformat{\subsubsection}%
	{\normalfont\normalsize\bfseries}{\thesubsubsection.}{1 ex}{}
\declaretheorem[within=section,name=Lemma]{Lem}
\declaretheorem[sibling=Lem,name=Definition]{Def}

\declaretheorem[sibling=Lem,name=Notation]{Not}
\declaretheorem[sibling=Lem,name=Proposition]{Prop}
\declaretheorem[sibling=Lem,name=Remark]{Rem}

\newcommand{\id}{\operatorname{id}}

\newcommand{\Var}{\operatorname{Var}}

\newcommand{\seq}{\operatorname{seq}}
\newcommand{\simplex}{\Delta}
\newcommand{\series}{\operatorname{series}}
\newcommand{\tpts}{\operatorname{time}}
\newcommand{\vals}{\operatorname{vals}}

\newcommand{\calH}{\mathcal{H}}

\newcommand{\calT}{\mathcal{T}}

\newcommand{\calV}{\mathcal{V}}
\newcommand{\calX}{\mathcal{X}}

\newcommand{\EE}{\ensuremath{\mathbb{E}}}

\newcommand{\NN}{\ensuremath{\mathbb{N}}}

\newcommand{\RR}{\ensuremath{\mathbb{R}}}

\newcommand{\bs}{\boldsymbol{s}}
\newcommand{\bt}{\boldsymbol{t}}

\newcommand{\bx}{\boldsymbol{x}}
\newcommand{\by}{\boldsymbol{y}}

\newcommand{\bSigma}{\boldsymbol{\Sigma}}

\newcommand{\bI}{\boldsymbol{I}}
\newcommand{\bK}{\boldsymbol{K}}

\newcommand{\bQ}{\boldsymbol{Q}}
\newcommand{\bT}{\boldsymbol{T}}
\newcommand{\bV}{\boldsymbol{V}}

\newcommand{\bX}{\boldsymbol{X}}

\makeatletter
\providecommand*{\diff}%
{\@ifnextchar^{\DIfF}{\DIfF^{}}}
\def\DIfF^#1{%
	\mathop{\mathrm{\mathstrut d}}%
	\nolimits^{#1}\gobblespace
}
\def\gobblespace{%
	\futurelet\diffarg\opspace}
\def\opspace{%
	\let\DiffSpace\!%
	\ifx\diffarg(%
	\let\DiffSpace\relax
	\else
	\ifx\diffarg\[%
	\let\DiffSpace\relax
	\else
	\ifx\diffarg\{%
	\let\DiffSpace\relax
	\fi\fi\fi\DiffSpace}
\makeatother

\def\approxd{\setbox0=\hbox{$\sim$}
	\setbox1=\hbox to \wd0{\hss.\hss}%
	\kern 4 pt\raise0.1ex\copy0\kern-\wd0\raise1.2ex%
	\copy1\kern-\wd0\copy1\kern 4 pt}

\def\dperp{ \mathrel{\kern0pt\vbox{\hbox to 0.8em{\hss	%
				\setbox0=\hbox{$\shortparallel$}\dp0=0pt\box0\hss}\hrule width 0.8em}}}

\usepackage{esint}
\makeatletter

\usepackage[utf8]{inputenc} %
\usepackage[T1]{fontenc}    %
\usepackage{url}            %
\usepackage{booktabs}       %
\usepackage{amsfonts}       %
\usepackage{nicefrac}       %
\usepackage{microtype}      %

\usepackage[backend=bibtex, natbib=true]{biblatex}
\title{Kernels for time series with irregularly-spaced multivariate observations}

\author{%
	Ahmed Guecioueur \thanks{Corresponding author: \href{mailto:ahmed.guecioueur@insead.edu}{ahmed.guecioueur@insead.edu}}  \\
	INSEAD \\
	 \\
	\And
	Franz J. Király\\
	University College London \& \\
	The Alan Turing Institute \\
}

\renewcommand{\headeright}{}
\renewcommand{\undertitle}{}

\begin{document}
		
\maketitle

\begin{abstract}
Time series are an interesting frontier for kernel-based methods, for the simple reason that there is no kernel designed to represent them and their unique characteristics in full generality.  Existing sequential kernels ignore the time indices, with many assuming that the series must be regularly-spaced; some such kernels are not even psd. In this manuscript, we show that a ``series kernel'' that is general enough to represent irregularly-spaced multivariate time series may be built out of well-known ``vector kernels''. We also show that all series kernels constructed using our methodology are psd, and are thus widely applicable. We demonstrate this point by formulating a Gaussian process-based strategy -- with our series kernel at its heart -- to make predictions about test series when given a training set. We validate the strategy experimentally by estimating its generalisation error on multiple datasets and  comparing it to relevant baselines. We also demonstrate that our series kernel may be used for the more traditional setting of time series classification, where its performance is broadly in line with alternative methods.
\end{abstract}


\section{Introduction}

Kernel-based methods are widely applicable. This is partly because a variety of algorithms can be neatly kernelised, ranging from ridge regression \cite{saunders1998ridge} to PCA \cite{scholkopf1998nonlinear} and multi-task methods \cite{argyriou2007multi}. It is also partly thanks to the availability of kernels for diverse data types, from unstructured data such as reals and bags of words to structured trees, graphs and DNA sequences \cite{shawe2004kernel}. In this manuscript, we aim to extend the reach of kernels to time series.

We are not the first to develop a kernel capable of representing \textit{sequential} real-valued data \cite{cuturi2011fast, cuturi2007kernel, gartner2004kernels, haussler1999convolution, kiraly2019, lodhi2002text} but the series kernel that we propose is the first to represent \textit{time series} in all their generality. Our method for building series kernels represents multivariate real-valued observations, their sequence orderings, and the sequential indices (such as timestamps) that they correspond to, all while allowing for irregularly-spaced observations without any specific preprocessing. The method is conceptually simple, builds upon widely-used vector kernels, and results in psd (positive semi definite) kernels.

To draw a contrast with prior work, points on a time series can be loosely represented by any unstructured kernel, though this loses the information given by the sequential ordering. String kernels \cite{lodhi2002text} add some sequential structure by counting common subsequences but may ignore the full sequential structure; it is also unclear how best to apply these to real-valued sequences like time series, and so they are typically used for biological string sequences. Convolution kernels \cite{haussler1999convolution} represent general composite objects in terms of their sub-parts; they have been extended to represent typed logics \cite{gartner2004kernels} and have inspired a general method to build sequential kernels out of individual kernels \cite{kiraly2019}. One popular technique is the global alignment kernel \cite{cuturi2007kernel, cuturi2011fast}, which captures the similarities between sequences based on the operations needed to map one sequence to another.  An alternative approach is to construct a sequential kernel out of distance or similarity measures for sequences \cite{abanda2019review,pree2014general}; often, though, care is needed to ensure the resulting sequential kernels are psd, and this is an ongoing area of research.

Time series are sequential in nature and so, in principle, can be represented by any of the aforementioned \textit{sequential} kernels, such as global alignment kernels or distance-based sequential kernels. However, none of these sequential kernels exploits the information given by the time indices; they act only on the sequence of observations. This means that it is not immediately clear how to deal with irregularly-spaced time series, including situations where some observations are missing. We show how to build \textit{series} kernels specifically for time series. Our series kernels handle irregularly-spaced time series natively and are simple to implement, as they are based on well-known vector kernels. They can also represent a sequence without an index if some non-informative time index is supplied in place of the missing index.

In many domains, time series are ubiquitous; practically, they may be irregularly-spaced \& multivariate, so our series kernel methodology is widely applicable. As a result, we hope that it will spur the development of further kernel-based methods on time series inputs, including multi-task learning, unsupervised learning (dimensionality reduction \& clustering) and other techniques.

As a demonstration of our series kernel's applicability, we consider the traditional task of time series classification. We also consider the more interesting task of multi-series regression, by formulating a Gaussian process-based model to take advantage of training series when making forecasts for test series. We report the results of predictive experiments on a variety of datasets.

In Section \ref{section:series_kernel} we show how to construct series kernels for irregularly-spaced multivariate time series, and prove that all such kernels are psd. We then validate series kernels experimentally in Section \ref{section:experiments}. We discuss our results and the importance of our new approach in Section \ref{section:discussion}. Section \ref{section:conclusion} concludes.

\section{Series kernels}
\label{section:series_kernel}

We will now elaborate a general technique to build series kernels -- that is, kernels that can represent multivariate irregularly-spaced time series -- out of well-known vector kernels. We prove that all such series kernels are psd. Throughout this section, we assume familiarity with the theory of kernels at the level of a standard textbook treatment  \cite{scholkopf2001learning,shawe2004kernel,vapnik1998statistical}.

Two prior studies make closely related, though distinct, contributions. \citet{lu2008reproducing} formulated a pairwise distance measure for irregularly-spaced time series; their distance measure may be viewed as being constructed from a series kernel that is itself constructed from a single underlying vector kernel, though the authors themselves did not pursue that line of reasoning. \citet{joel} implemented a Gaussian/RBF kernel that applied kernel ridge regression to smooth its inputs.

\subsection{Notation for series and sequences}

\begin{Def}
	\label{def:series}
	We first introduce some auxiliary notation:
	\begin{itemize}
		\item[(i)] Let $\calX$ be any set. We denote by $\seq (\calX)$ the set of arbitrary tuples with entries in $\calX$, i.e, $\seq (\calX):= \{(x_1,...,x_m)\;;\;m \in \NN, x_i\in \calX\}$. We call elements of $\seq(\calX)$ sequences (with values) in $\calX$. For $\bx\in\seq(\calX)$, we denote by $\ell(x)$ the length of $\bx$, as a tuple.
		\item[(ii)] Let $\calT\subseteq\RR$. We denote by $\simplex (\calT)$ the ascending sequences in $\calT$, i.e., $\bt\in \simplex (\calT)$ iff $\bt\in\seq(\calT)$ and $\bt_i\le \bt_j$ for all $i\le j$.
		\item[(iii.a)] Let $\calX$ be any set, let $\calT\subseteq \RR$. We denote $\series(\calX,\calT):=\{(\bx,\bt)\in \seq(\calX)\times \simplex(\calT) \;;\; \ell(\bt) = \ell(\bx)\}$ and call elements of $\series (\calX,\calT)$ (discrete) time series in $\calX$.
		\item[(iii.b)] for $\bs = (\bx,\bt)\in \series (\calX,\calT)$, we write abbreviatingly $\vals(\bs) := \bx, \tpts(\bs) := \bt)$. By abuse of notation, we will also consider $\vals(\bs),\tpts(\bs)$ column vectors. %
	\end{itemize}
\end{Def}

\begin{Not}
	Let $k:\calX\times\calX\rightarrow \RR$ a kernel for an arbitrary set $\calX$, let $\bx,\by\in\seq(\calX)$, where $\ell(\bx) = m, \ \ell(\by) = n$. We will abbreviatingly denote by $k(\bx,\by)$ the $(m \times n)$ real matrix with entries $k(\bx,\by)_{ij} = k(\bx_i,\by_j)$.
\end{Not}

\subsection{Series kernels and their construction}
\label{section:construction}

We are now ready to define our object of interest: a kernel for potentially unequally-spaced time series. For exposition, we define it first for univariate time series.

\begin{Def}
	\label{def:series_kernel}
	Let $\calX,\calT \subseteq \RR$. Let $k', k'':\calT\times\calT\rightarrow \RR$ be psd kernels, assume $k''$ is in addition non-singular (= $k''$-matrices are always invertible). Define
	\begin{align*}
	k : \ & \series(\calX, \calT) \times \series(\calX, \calT) \to \RR  \label{eqn:series_kernel_types}\\
	& ((\bx,\bt),(\by,\bs)) \mapsto \bx^\top \cdot k'(\bt,\bt)^{-1}\cdot k''(\bt,\bs)  \cdot k'(\bs,\bs)^{-1} \cdot \by ,
	\end{align*}
	where we consider $k$ dependent on $k',k''$, albeit not explicit in notation.
\end{Def}

An intuitive interpretation of $k$, in terms of $k'$ and $k''$, is as follows: first, interpolate $(\bx,\bt)$ and $(\by,\bs)$ separately, by Gaussian process regression (or, equivalently, kernel ridge regression). Then compute some inner product (such as the L2 product) between the resultant functions. Finally, replace the appearing psd matrices with general kernel matrices to obtain the general expression. The following proposition and its proof highlight that all kernels $k$ may be obtained in such a way:

\vspace{0.2cm}

\begin{Prop}
\label{prop:psd}
Consider the kernel $k$, as defined in Definition~\ref{def:series_kernel}. Then the following hold:
	\begin{itemize}
	\item[(i)] $k$ is a psd kernel.
	\item[(ii)] $k(\bx,\by) = \left\langle f,g\right\rangle_\calH$ for some suitable Hilbert space $\calH$  of functions $\RR\rightarrow \RR$, and $f,g$ being the kernel ridge regression (= GP regression) interpolates of $\bx,\by,$ for a suitably chosen kernel.
	That is, there exist $\lambda\in \RR^+$ and a kernel $\tilde{k}$ such that $f(z) = \vals(\bx)^\top (\tilde{k}(\bt,\bt) + \lambda I)^{-1} \tilde{k}(\bt, z)$ and $g(z) = \vals(\by)^\top (\tilde{k}(\bs,\bs) + \lambda I)^{-1} \tilde{k}(\bs, z)$, where $\bt = \tpts(\bx)$ and $\bs = \tpts(\by)$.
\end{itemize}
\end{Prop}
\begin{proof}
	(i) is implied by (ii), since (ii) implies that any kernel matrix in $k$ is a Gram matrix, hence psd.\\
	(ii) Since $k''$ is psd, by the Moore-Aronszajn theorem there exists an RKHS $\calH$ and a reproducing kernel $\tilde{k}:\calT\times\calT\rightarrow \RR$ such that $k''(x,y) = \left\langle \tilde{k}(x,.),\tilde{k}(y,.)\right\rangle_\calH$ for all $x,y\in\calT$.
	Thus,
	{%
		\begin{align*}
			k(\bx,\by)  &= \vals(\bx)^\top \cdot k'(\bt,\bt)^{-1}\cdot \left\langle \tilde{k}(\bt,.), \tilde{k}(.,\bs)\right\rangle_\calH  \cdot k'(\bs,\bs)^{-1} \cdot \vals(\by) \\
			&=  \left\langle \vals(\bx)^\top \cdot k'(\bt,\bt)^{-1}\cdot \tilde{k}(x,.)\;,  \;\tilde{k}(.,\bs)  \cdot k'(\bs,\bs)^{-1} \cdot \vals(\by) \right\rangle_\calH 
		\end{align*}
	}%
	by linearity of the inner product. This is equal to $\left\langle f,g\right\rangle_\calH$ as in the statement (ii), with $\lambda = 0$, and the remaining symbols chosen in accordance with the statement.
\end{proof}

Conversely, any kernel constructed as in statement (ii) of Proposition~\ref{prop:psd} (with arbitrary $\lambda$) will satisfy the properties of Definition~\ref{def:series_kernel}. Thus, Definition~\ref{def:series_kernel} defines precisely those kernels on time series that can be constructed according to the intuition that we provided earlier, even though it doesn't define them explicitly through that particular intuitive algorithm. In fact, actually computing a time series kernel using that same intuition (first performing KRR, and then computing the inner product) would \emph{not} be naively possible, as it would require the handling of infinite objects in finite memory. On the other hand, the simple computation implied by Definition~\ref{def:series_kernel} is entirely practicable, as a product of finite size matrices.

Also, note that even though the value vectors appear only linearly, by Proposition~\ref{prop:psd}~(ii) this arises from interpolates subjected to arbitrary inner products, in arbitrary Hilbert spaces, by choosing a different $k''$. Therefore, linearity in the value vectors is not a substantial restriction in expressiveness, perhaps somewhat surprisingly.

One interesting ability of our series kernels is that they can be applied to \textit{any} pair of time series inputs, irrespective of what observations the time series have in common; indeed, the extreme case of a pair of time series that have no common observations at all, $\bt \cap \bs = \{\}$, is automatically handled. At the same time, no observations are ever ignored or dropped. Intuitively, this is thanks to the interpretation of KRR/GP regression as an interpolating or smoothing technique.

Multivariate generalizations, i.e. to the case $\calX\subseteq \RR^d$, may be obtained in complete analogy by considering Hilbert space products of functions $\calX\rightarrow \RR$ in Proposition~\ref{prop:psd}. A general form is as follows:

\vspace{0.2cm}

\begin{Def}
	\label{def:series_kernel_multivariate}
	Let $\calX\subseteq \RR^d$, let $\calT \subseteq \RR$. Let $k', k'':\calT\times\calT\rightarrow \RR$ be psd kernels, assume $k''$ is in addition non-singular (= $k''$-matrices are always invertible). Define
	\begin{align*}
	k : \  & \series(\calX, \calT) \times \series(\calX, \calT) \to \RR\\
	& ((\bx,\bt),(\by,\bs))  \mapsto  \sum_{i,j=1}^d  w_{ij}\cdot \left( \bx^{(i)} \right)^\top \cdot k'(\bt,\bt)^{-1}\cdot k''(\bt,\bs)  \cdot k'(\bs,\bs)^{-1} \cdot \by^{(j)},
	\end{align*}
where $\bx^{(i)},\by^{(i)}$ denotes the $i$-th resp.~$j$-th coordinate series of $\bx$ resp.~$\by$.
\end{Def}

\pagebreak

This could be further generalized to the case where the different coordinates are observed at different times. For simplicity of exposition, we will deal with the univariate case in subsequent discussions, unless stated otherwise.

\subsection{Computational considerations}
\label{section:computation}

In kernel learning, primary interest lies in the setting where we have multiple input time series, usually assumed to arise from an i.i.d.~process. Formally, we will assume an input of $N$ series $\bx_1,\dots,\bx_N$, jointly referred to as an $N$-tuple $\bX\in\series(\calX, \calT)^N$ (where $\bx_i=\bX_i$ are used interchangeably). Further data may be present as ``test set'' features, as an input of $M$ series $\bx_1^*,\dots,\bx_M^*$, jointly referred to as an $M$-tuple $\bX\in\series(\calX, \calT)^M$.

In various kernel and Gaussian process algorithms, one is usually interested in computing:
\begin{itemize}
	\item[(a)] kernel matrices $k(\bX,\bX)$,
	\item[(b)]  cross-kernel matrices $k(\bX,\bX^*)$, of which (a) is the special case where $\bX^*=\bX$, and
    \item[(c)] inverse kernel matrices $k(\bX,\bX)^{-1}$.
\end{itemize}

We will now make a number of general observations in aid of the above:

\begin{Rem}
Denote $\bt_i:=\tpts(\bx_i)$ and $\bt^*_i:=\tpts(\bx^*_i)$. Let $\alpha \in [2,3]$ be the scaling constant of a chosen matrix inversion algorithm, such that inverting an $(n\times n)$ matrix costs $O(n^\alpha)$. Known algorithms have $\alpha\ge 2.3$ \cite{coppersmith1990matrix}, while it is occasionally conjectured that $\alpha =2$ is achievable \cite{cohn2005group}. Let $D$ be the (asymptotically) typical length (using the appropriate average) of $\tpts(\bx_i)$ resp.~$\tpts(\bx^*_i)$. Further assume that computing $k'$ and $k''$ cost $O(1)$, e.g. for common kernels on the reals, or when they have already been computed. We also assume $M\le N$, otherwise we switch $\bX$ and $\bX^*$.
\begin{itemize}
	\item[(i)] Naive entry-by-entry computation of (b) $k(\bX,\bX^*)$ costs $O(MN (D^\alpha+MD(D+N)).$ Through pre-computing the inverses $k(\bt_i,\bt_i)^{-1}$ and $k(\bt_i^*,\bt_i^*)^{-1},$ this becomes $O((M+N)D^\alpha + M^2ND(D+N))$ which is strictly less in terms of $M,N,D$.
	\item[(ii)] If all $\tpts(\bx_i)$ are identical to some $\bt$, and all $\tpts(\bx_i^*)$ are identical to some $\bt^*$, an elementary computation shows that $k(\bX,\bX^*) = \bX^\top \cdot k'(\bt,\bt)^{-1}\cdot k''(\bt,\bt^*)  \cdot k'(\bt^*,\bt^*)^{-1} \cdot \bX^*$, where $\bX$ and $\bX^*$ are arranged in $(\ell(\bt)\times N)$ and $(\ell(\bt^*)\times M)$ matrices. Computation using this formula costs only $O(D^\alpha + MD(D+N))$, which is substantially less than the complexity in (i). Therefore, it may be beneficial to take unions or intersections over time points, if there is only small discrepancy between the $\bt_i$ or $\bt_i^*$ in (i). A more sophisticated approach is computing inverses of slightly discrepant $k(\bt_i,\bt_i)$ from each other, via the Sherman-Morrison-Woodbury formulae. Similarly, if $\bt_i$ differ only across blocks (of indices $i$), than block-inversion may be preferable.
\item[(iii)] A naive computation of inverse kernel matrices $k(\bX,\bX)^{-1}$, via computing the matrix and then its inverse from it, costs whatever complexity of (a) plus $O(N^\alpha)$. In the case of all $\tpts(\bx_i)$ identical, this is $O(D^\alpha + N^\alpha + MD(D+N))$ etc. Alternatively, an elementary computation shows that
    $k(\bX,\bX)^{-1} = \bX^{+\top} \cdot k'(\bt,\bt)\cdot k''(\bt,\bt)^{-1} \cdot k'(\bt,\bt) \cdot \left(\bX^*\right)^{+},$ where $.^+$ denotes the Moore-Penrose pseudo-inverse. Using this formula, the cost becomes $O(D^\alpha + MD(D+N))$, which is an improvement whenever $D\le N$, e.g. in the regime where $D$ does not grow as the number of data points $N$ grow.
\item[(iv)] The formulae can be further simplified when taking $k''=k'$, but empirically we have found this to be a bad choice. Instead, it seems more advisable to choose $k'(x,y) = \tilde{k}'(x,y) + \lambda \cdot \delta (x,y),$ where $\delta$ is the Kronecker delta function, and distinct real kernels $\tilde{k}', k''$, where $\lambda\in \RR^+$ is a regularization parameter, then tune $\tilde{k}', k'',\lambda$ considered as independent parameters. This is intuitively sensible via Proposition~\ref{prop:psd}, since $k''$ encodes similarity between the series, while $k'$ encodes temporal similarity between values within a series, and there is no a-priori reason why these (or the similarities they encode) should agree.
\item[(v)] Any given (symmetric positive definite) Gram matrix is amenable to spectral decompositions of the form $\bK = \bQ \bV \bQ^T$, where $\bV$ is a diagonal matrix of eigenvalues $v_1,\ldots,v_n$. This enables us to eliminate all re-inversions during $\lambda$'s hyperparameter tuning, while also avoiding the re-computation of $\bQ$ values associated with different implicit parameters of $k$ as those are being tuned: $ (\bK + \lambda \bI )^{-1} = \bQ (\bV + \lambda \bI)^{-1} \bQ^T = \bQ \big( \operatorname{diag}_{i=1,\ldots,n}  \{ v_i + \lambda \} \big)^{-1} \bQ^T = \bQ \ \operatorname*{diag}_{i=1,\ldots,n}  \{ ( v_i + \lambda )^{-1} \} \ \bQ^T $
\end{itemize}
\end{Rem}

\section{Experiments}
\label{section:experiments}

In this section, we demonstrate how our series kernel methodology may be applied to predictive tasks when one's dataset consists of multiple time series. Rather than engaging in a horse-race of techniques, our aim here is to confirm that our methodology is applicable to a variety of settings. Throughout, we use open-source software \cite{pysf,tslearn,sktime,scikit-learn} to conduct our experiments.

Much of the machine learning time series literature reports experimental results on the classification task, so we likewise begin with this. We then go on to use our series kernel methodology in a multi-series regression problem; that is, making forecasts for some test time series given a training set of other time series. We believe that this regression example is a natural one for time series data, and has not been previously considered by the sequential kernels literature. We would like to emphasise, though, that our series kernel methodology may be applied to machine learning problems beyond supervised learning.

\subsection{SVM-based classification}
\label{section:experiment_classification}

Time series classification has been extensively studied, and one generally-accepted conclusion in the literature is that a 1-NN classifier using the Dynamic Time Warping (DTW) distance is a suitable benchmark for this task because of its good performance \cite{bagnall16bakeoff}. Among sequential kernels, the global alignment kernel \cite{cuturi2011fast} is well-known, so we incorporate it into an SVM classifier and adopt that as our second benchmark.

We conduct this classification experiment on the \textit{Arrowhead}, \textit{Gunpoint} and \textit{Italy power demand} datasets of univariate time series, each available at the UCR/UEA repository \cite{bagnalluea}. We use the training/test splits provided to us by that repository, as is common in the literature. Table \ref{tbl:datasets_classification} summarises these three datasets.

\begin{table}[ht]
	\centering
	\small
	\bgroup
	\def\arraystretch{1.2} %
	\begin{tabular}{ | l | r r r r l |}
		\hline
		\textbf{Classification dataset} & \textbf{Training series} & \textbf{Test series} & \textbf{Class labels} & \textbf{Timestamps per series} & \textbf{Domain} \\
		\hline
		Arrowhead                       & 36                       & 175                  & 3                     & 176                                   & Image outline   \\
		Gunpoint                        & 50                       & 150                  & 2                     & 150                                   & Motion          \\
		Italy power demand              & 67                       & 1029                 & 2                     & 24                                    & Sensor reading \\
		\hline
	\end{tabular}
	\normalsize
	\vspace{0.25cm}
	\caption{Summary of the datasets used in the time series classification experiments.}
	\label{tbl:datasets_classification}
	\egroup
\end{table}

\begin{table}[ht]
	\centering
	\small
	\bgroup
	\def\arraystretch{1.2} %
	\begin{tabular}{| l | c c c |}
		\hline
		\textbf{Prediction strategy} &  \textbf{Arrowhead} &  \textbf{Gunpoint} & \textbf{Italy power demand} \\
		\hline
		Series kernel SVM &  77.71\% $\pm$ 3.15\% &  92.00\% $\pm$ 2.22\% &   95.34\% $\pm$ 0.66\% \\
		\hline
		DTW distance 1-NN          &  70.29\% $\pm$ 3.46\% &  90.67\% $\pm$ 2.38\% &   95.04\% $\pm$ 0.68\% \\
		Global alignment kernel SVM           &  82.29\% $\pm$ 2.89\% &  98.67\% $\pm$ 0.94\% &   96.11\% $\pm$ 0.60\% \\
		\hline
	\end{tabular}
	\normalsize
	\vspace{0.25cm}
	\caption{Time series classification experimental results. We report classification accuracy values $\pm$ 1 s.e. }
	\label{tbl:experiments_classification}
	\egroup
\end{table}

Within the training set, we conduct 5-fold cross validation to tune the hyperparameters of our strategies. The 1-NN DTW kernel has no such hyperparameters, but the other two kernel-SVM classifiers each have one hyperparameter associated with the SVM (a regularization parameter $C$, which we tune over parameters sampled from a log grid) and parameters that are specific to the kernels embedded within the SVMs. All these multiple hyperparameters are tuned jointly. The global alignment kernel has a kernel bandwidth parameter \cite{cuturi2011fast} that we tune over a grid of linearly-sampled values. 

We take a careful approach to tuning the series kernel, which we implement in the form given by Proposition~\ref{prop:psd}~(ii). It thus requires a second regularization parameter $\lambda$, which we tune over parameters sampled from a log grid. We also define additional hyperparameters that specify how the constituent vector kernels ($k', k''$) are constructed: one parameter is for the type of the between-series vector kernel (possible types: linear, RBF or Laplacian) and a similar parameter is for the type of the within-series vector kernel (same possible types). There are also there are two additional $\gamma$ ($=\sfrac{1}{2\sigma^2}$) parameters that we tune over, one for each of ($k', k''$) when they are RBF or Laplacian vector kernels; the linear vector kernel has no such parameter. In this way, our approach is to treat the form of the constituent vector kernels ($k', k''$) as parameters to tune over. An alternative approach would have been to initialise a variety of series kernels based on combinations of constituent vector kernels (eg. linear-Laplacian, RBF-RBF, RBF-linear, and so on) according to our methodology, and treat those as separate predictors in individual experiments; we believe our approach is neater and allows us to evaluate our series kernel construction methodology as an overall prediction strategy.

Table \ref{tbl:experiments_classification} presents the out-of-sample classification accuracy scores for this experiment. The first row shows results for an SVM classifier that uses a series kernel to calculate similarities between entire time series, rather than just individual points. The third row shows results for a similar SVM-based strategy, but where the kernel used is a global alignment kernel instead. Each kernel is embedded in an SVM classifier. The tuning is as described above. Standard errors are calculated using the jackknife. The scores indicate that the series kernel-based SVM strategy always outperforms the 1-NN DTW baseline, though it never ranks as the best classifier. We make no claims about statistical significance with this rather limited set of estimates. Nevertheless, these limited results appear to suggest that series kernels constructed according to our methodology perform broadly in line with related approaches. A more rigorous benchmarking exercise along the lines of \citet{bagnall16bakeoff} is needed to draw stronger conclusions.

A final point of interest is that we have deliberately restricted ourselves to univariate datasets with regularly-spaced samples in this setting, for two reasons: (i) to use datasets that have been well-studied in the time series classification literature, and (ii) to avoid excluding interesting techniques that are designed specifically for sequences rather than general time series. As more techniques are developed to deal with general time series, rather than sequences, it would be interesting to revisit these classification experiments in the future. The regression experiment we now present in Section \ref{section:experiment_regression} takes a step in this direction, and away from the standard setting.

\subsection{Gaussian process-based regression}
\label{section:experiment_regression}

Since our series kernel is psd, it lends itself to being interpreted as a covariance matrix between time series, and we can thus adopt a Gaussian process model of the data in the style of \citet{RasmussenWilliams} pp. 15-17. More precisely, under the assumptions of a zero-mean joint Gaussian prior on training series observations and noisy test series observations, whose joint covariance matrix can be partitioned into Gram matrices of series kernel evaluations, the mean of the conditional posterior of the joint distribution can be given in closed form. It is this posterior mean that shall be our predicted value for the test time series. It is worth emphasising that in our model each data point is, in fact, an entire time series segment.

We will conduct our Gaussian Process regression (GPR) benchmarking on four datasets consisting of multiple time series. Fig. \ref{fig:ramsay_growth_data_plots} plots the \textit{Berkeley growth} dataset of \citet{tuddenham1954physical}. This consists of height readings of 93 children taken at irregular intervals (quarterly then semi-annual) over 18 years. The 93 series are each univariate. Fig. \ref{fig:ramsay_canadian_weather_data_plots} plots the \textit{Canadian weather} dataset of \citet{FDA}. This consists of average daily temperature \& precipitation readings from 35 weather stations. The 35 series are each bivariate. Both datasets are available online\footnote{See James Ramsey’s website at \url{http://www.psych.mcgill.ca/misc/fda/downloads/FDAfuns/}}. Importantly, these two datasets possess some characteristics that are relevant to practical applications -- the former is irregularly-spaced and the latter is multivariate -- and we will argue shortly that it is important to consider such datasets when designing methods for time series.

For the sake of comparison, in this experiment we will also make use of two additional datasets of univariate series, each available at the UCR/UEA repository \cite{bagnalluea}: \textit{ECG200} (200 cardiac cycles) and \textit{Starlight} (periodic photometric readings, of which we sample 100 series). Table \ref{tbl:datasets_regression} summarises all four datasets.

\begin{figure}[ht]
	\centering
	\vspace*{-0.25cm}
	\includegraphics[width=\myseriesplotfigwidth, angle=0]{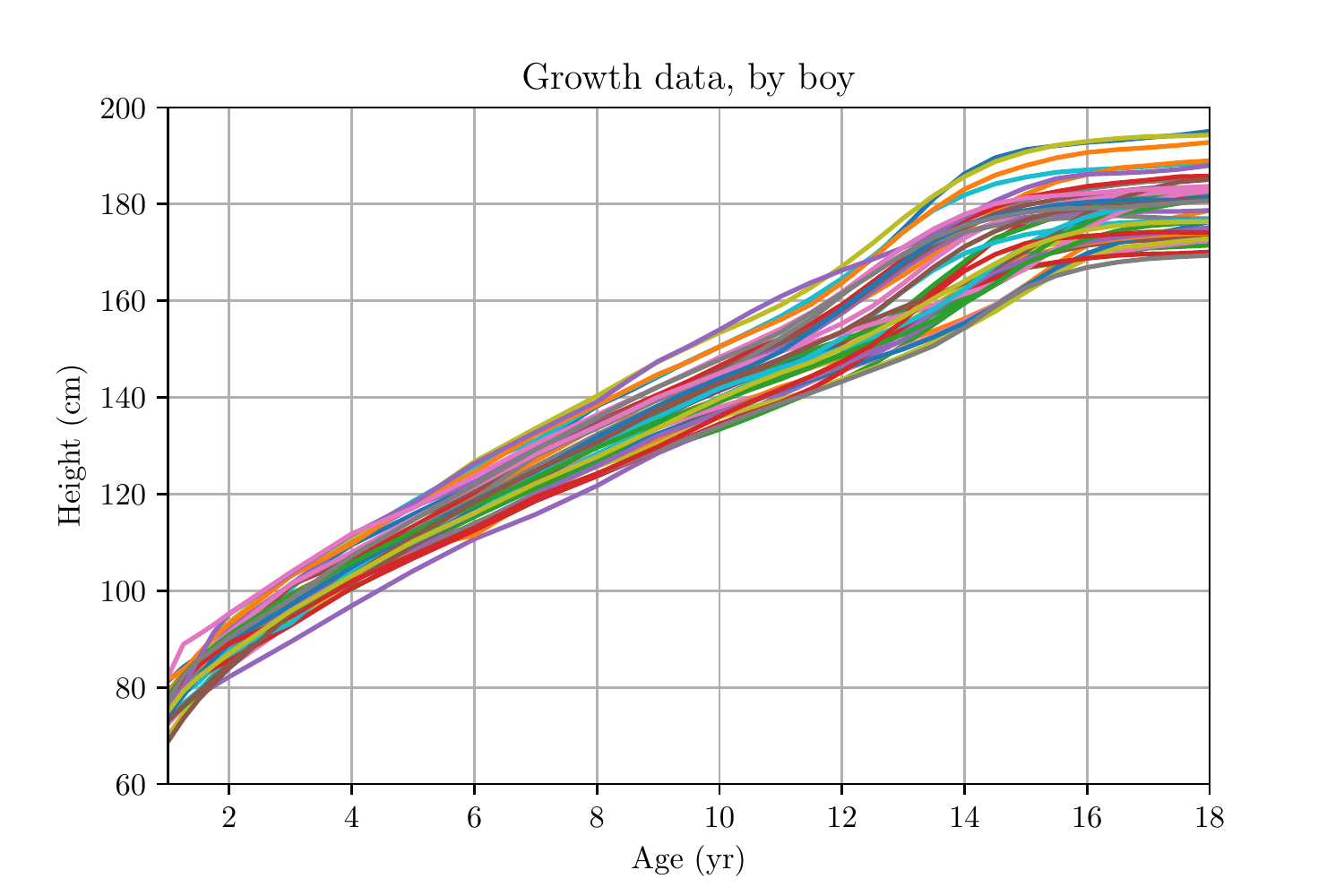}
	\includegraphics[width=\myseriesplotfigwidth, angle=0]{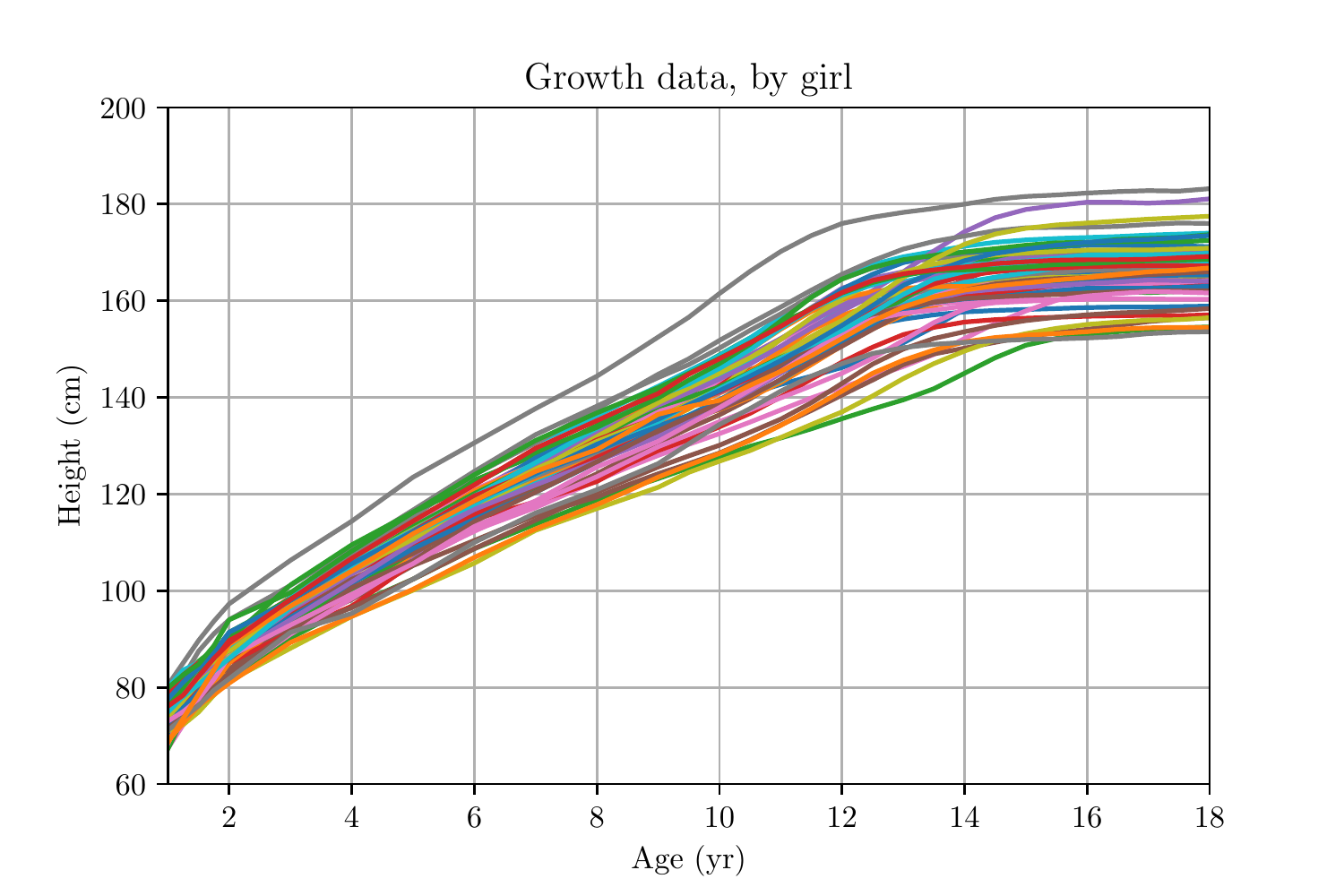}
	\vspace*{-0.25cm}
	\caption{Series plots of the Berkeley growth data. Each series is labelled as a boy or girl, so the data has been divided into 2 corresponding plots for legibility. Height is the only time-indexed variable.}
	\label{fig:ramsay_growth_data_plots}
\end{figure}

\begin{figure}[ht]
	\centering
	\vspace*{-0.5cm}
	\includegraphics[width=\myseriesplotfigwidth, angle=0]{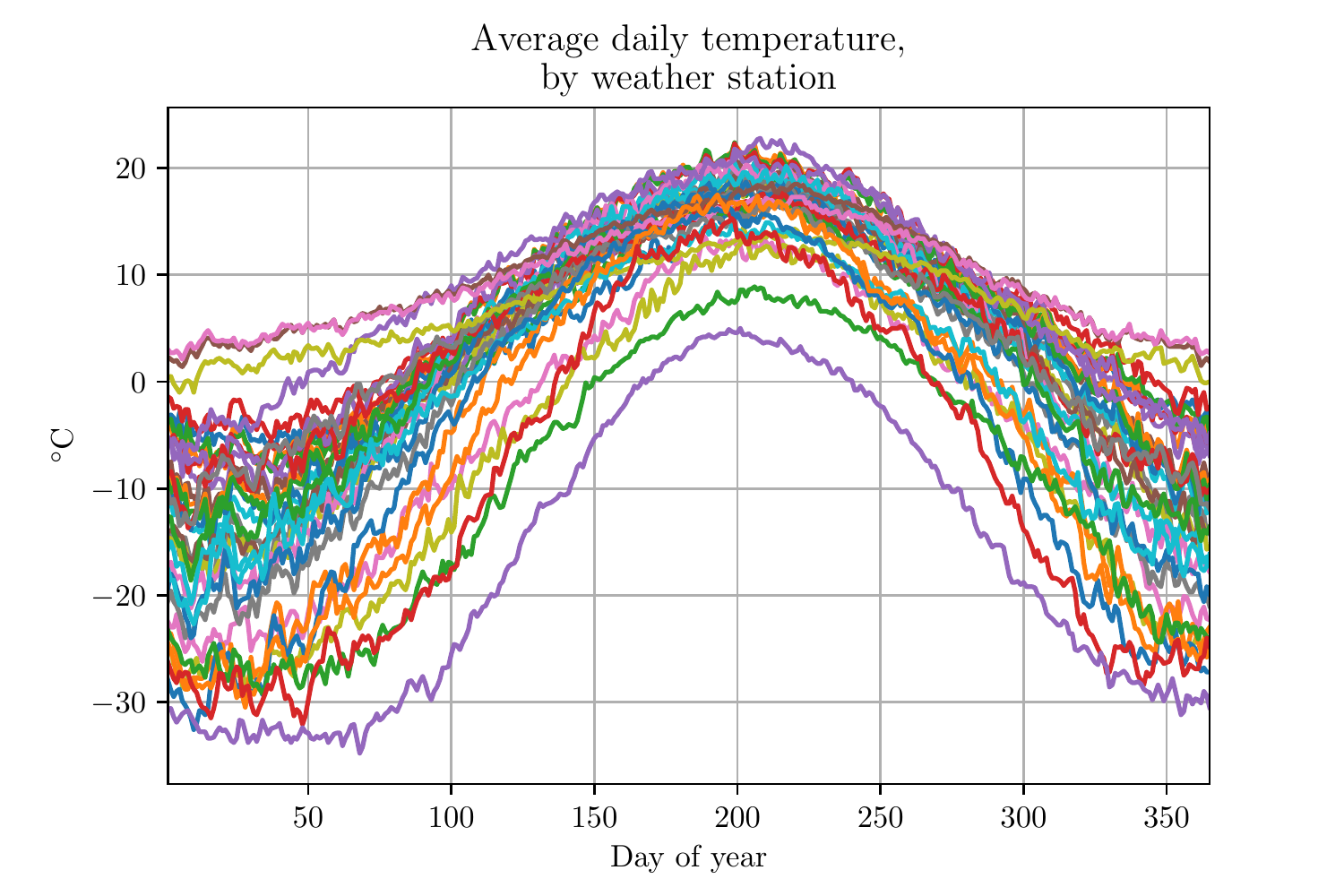}
	\includegraphics[width=\myseriesplotfigwidth, angle=0]{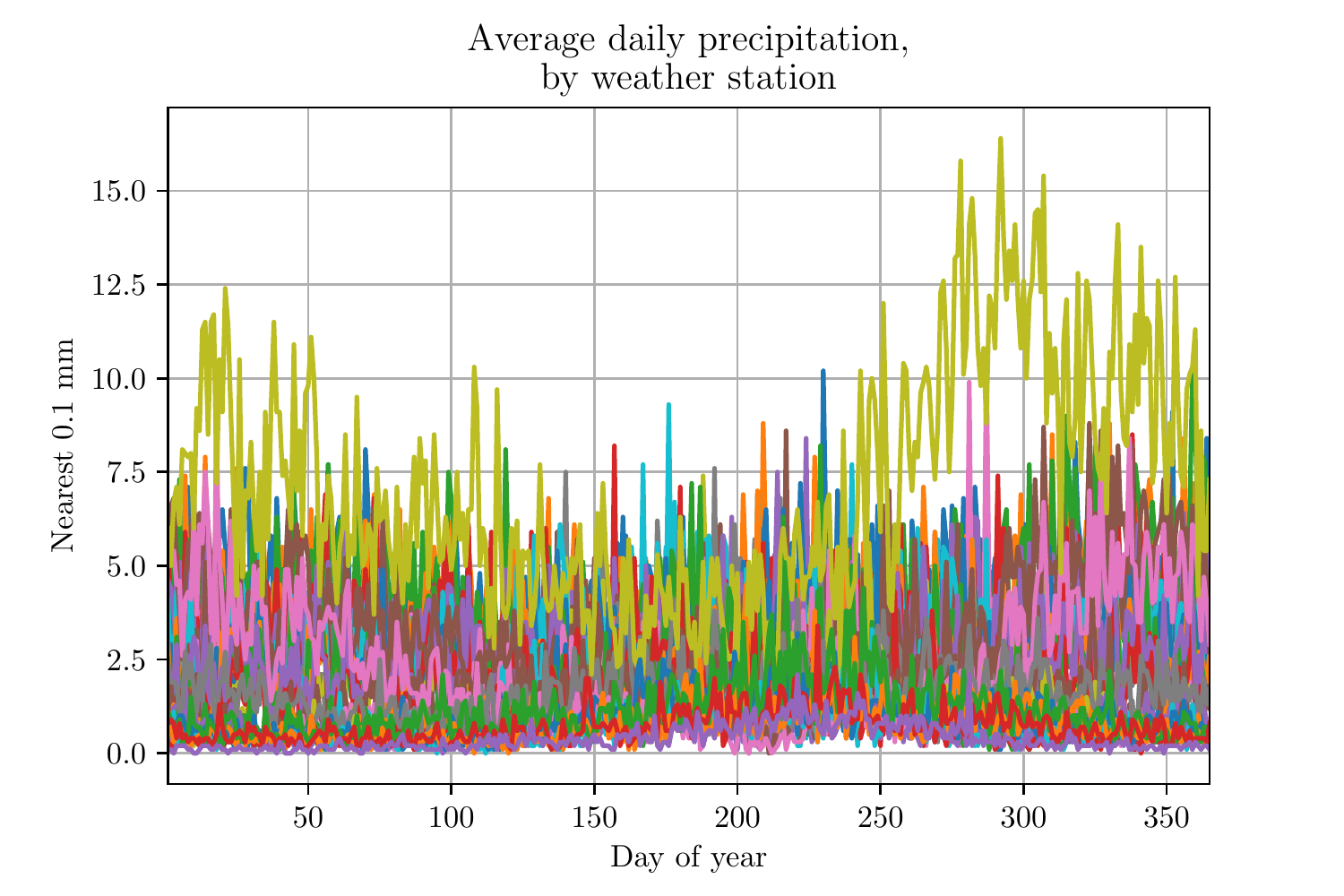}
	\vspace*{-0.25cm}
	\caption{Series plots of the Canadian weather data, with each of the 2 variables (average temperature \& average precipitation) of the bivariate series plotted separately.}
	\label{fig:ramsay_canadian_weather_data_plots}
\end{figure}

\begin{table*}[ht]
	\centering
	\small
	\bgroup
	\def\arraystretch{1.2} %
	\begin{tabular}{ | l | r r r l l | }
		\hline
		\textbf{Regression dataset} & \textbf{Series} & \textbf{Training timestamps} & \textbf{Test timestamps} & \textbf{Irregularly-spaced}   & \textbf{Multivariate} \\
		\hline
		Berkeley growth              & 93              & First 22                      & Last 9                   & Yes: quarterly \& semi-annual & No                    \\
		Canadian weather             & 35              & First 300                     & Last 65                  & No                            & Yes: 2 variables      \\
		ECG200                       & 200             & First 48                      & Last 48                  & No                            & No                    \\
		Starlight (sampled)                    & 100   & First 500                      & Last 500                  & No                            & No                   \\
		\hline
	\end{tabular}
	\normalsize
	\caption{Summary of the datasets used in the time series regression experiments.}
	\label{tbl:datasets_regression}
	\egroup
\end{table*}

\begin{table*}[ht]
	\centering
	\small
	\bgroup
	\def\arraystretch{1.2} %
	\begin{tabular}{|l|rrrrr|}
		\hline
		& & \multicolumn{2}{c}{\textbf{Canadian weather}}  & &  \\
		\cline{3-4}
		\textbf{\thead{Prediction strategy}} & \textbf{\thead{Berkeley \\ growth}} & \textbf{\thead{precipitation \\ variable}} & \textbf{\thead{temperature \\ variable}} & \textbf{\thead{ECG200}} & \textbf{\thead{Starlight}} \\
		\hline
		Series kernel GPR & 5.44 $\pm$ 0.68 *** & 1.09 $\pm$ 0.27  & 1.91 $\pm$ 0.36 *** & 0.29 $\pm$ 0.03 *** & 0.26 $\pm$ 0.06 *** \\
		\hline
		Global alignment kernel GPR & 4.16 $\pm$ 0.68 *** & 1.75 $\pm$ 0.68 & 3.82 $\pm$ 1.49 ** & 0.25 $\pm$ 0.03 *** & ---  \\ 
		ARIMA & 13.77 $\pm$ 1.87 & 1.28 $\pm$ 0.36 & 6.05 $\pm$ 1.88 & 0.95 $\pm$ 0.07 & 1.09 $\pm$ 0.10 \\
		ElasticNet & 10.29 $\pm$ 0.82 & 1.38 $\pm$ 0.41 & 12.07 $\pm$ 1.16 & 0.94 $\pm$ 0.05 & 1.09 $\pm$ 0.10 \\
		Linear & 14.88 $\pm$ 1.04 & 1.85 $\pm$ 0.41 & 30.38 $\pm$ 2.56 & 2.56 $\pm$ 0.23 & 2.59 $\pm$ 0.25 \\
		\hline
		Baseline timestamp means & 9.59 $\pm$ 1.38 & 2.28 $\pm$ 0.56 & 8.61 $\pm$ 1.34 & 0.43 $\pm$ 0.03 & 0.59 $\pm$ 0.06 \\
		Baseline zero values & 169.78 $\pm$ 1.58 & 3.15 $\pm$ 0.96 & 10.72 $\pm$ 1.91 & 0.59 $\pm$ 0.03 & 0.78 $\pm$ 0.04 \\
		Baseline linear interpolator & 8.07 $\pm$ 0.74 & 1.35 $\pm$ 0.34 & 9.97 $\pm$ 1.04 & 0.87 $\pm$ 0.08 & 1.14 $\pm$ 0.07 \\
		Baseline series means & 46.45 $\pm$ 0.83 & 1.38 $\pm$ 0.41 & 12.07 $\pm$ 1.16 & 0.94 $\pm$ 0.05 & 1.09 $\pm$ 0.10 \\
		\hline
	\end{tabular}
	\normalsize
	\caption{Time series regression experimental results. We report estimated generalisation errors (RMSE) $\pm$ 1 s.e. Stars indicate that our t-test's null hypothesis (of the given strategy's performance being indistinguishable from or inferior to that of the best baseline's) was rejected at a given level of confidence: ***=99\%, **=95\%, *=90\%.}
	\label{tbl:experiments_combined_mean_plus_se}
	\egroup
\end{table*}

For simplicity, we fix the test timestamps to be the terminal subset of observations: the last 9, 65, 48 and 500 observations of the Berkeley, Canadian, ECG200 and Starlight datasets, respectively. In this sense, our regression task is actually a forecasting task, which opens the door to using some methods that are specifically designed for forecasting: we thus include an ARIMA model among our prediction strategies.

Once again, we discuss our tuning procedures in some detail. We set up our series kernel and global alignment kernel to be tuned in exactly the same way as we described in Section \ref{section:experiment_classification}. The main difference is that we now have a GPR-based prediction strategy rather than an SVM-based strategy, so we now have a regularization parameter $\lambda$ to tune over instead of $C$, as we had in our classification setup. All other aspects of kernel construction and hyperparameter tuning are identical. The two remaining predictors that require tuning are ARIMA (which has that statistical model's standard $p,d,q$ integer hyperparameters) and ElasticNet (which has an $\alpha$ penalisation hyperparameter that we sample from a log grid, and a second hyperparameter that controls the fraction of mixing between L1 and L2 penalties, which we sample linearly).

We conduct nested 5-fold cross-validation so as to avoid any potential bias in our estimates \cite{varma,cawley}: the inner loop is for the hyperparameter tuning discussed above, and the outer loop is to provide the resampling that is necessary to enable the estimation of the generalization error. The cross-validation splits are along series. Fig. \ref{fig:stylised_multi_series_index_sets} is a stylized illustration of our goal for one example fold (with an equal split for legibility): learning from some training set of series $X_{\calT}$ to predict $X_{\calV} (\bT_{\calV})$; i.e. the values of some set of test series $X_{\calV}$ at prediction times $\bT_{\calV}$.

\begin{figure}[ht]
	\centering
	\vspace*{-0.25cm}
	\includegraphics[width=8.0cm]{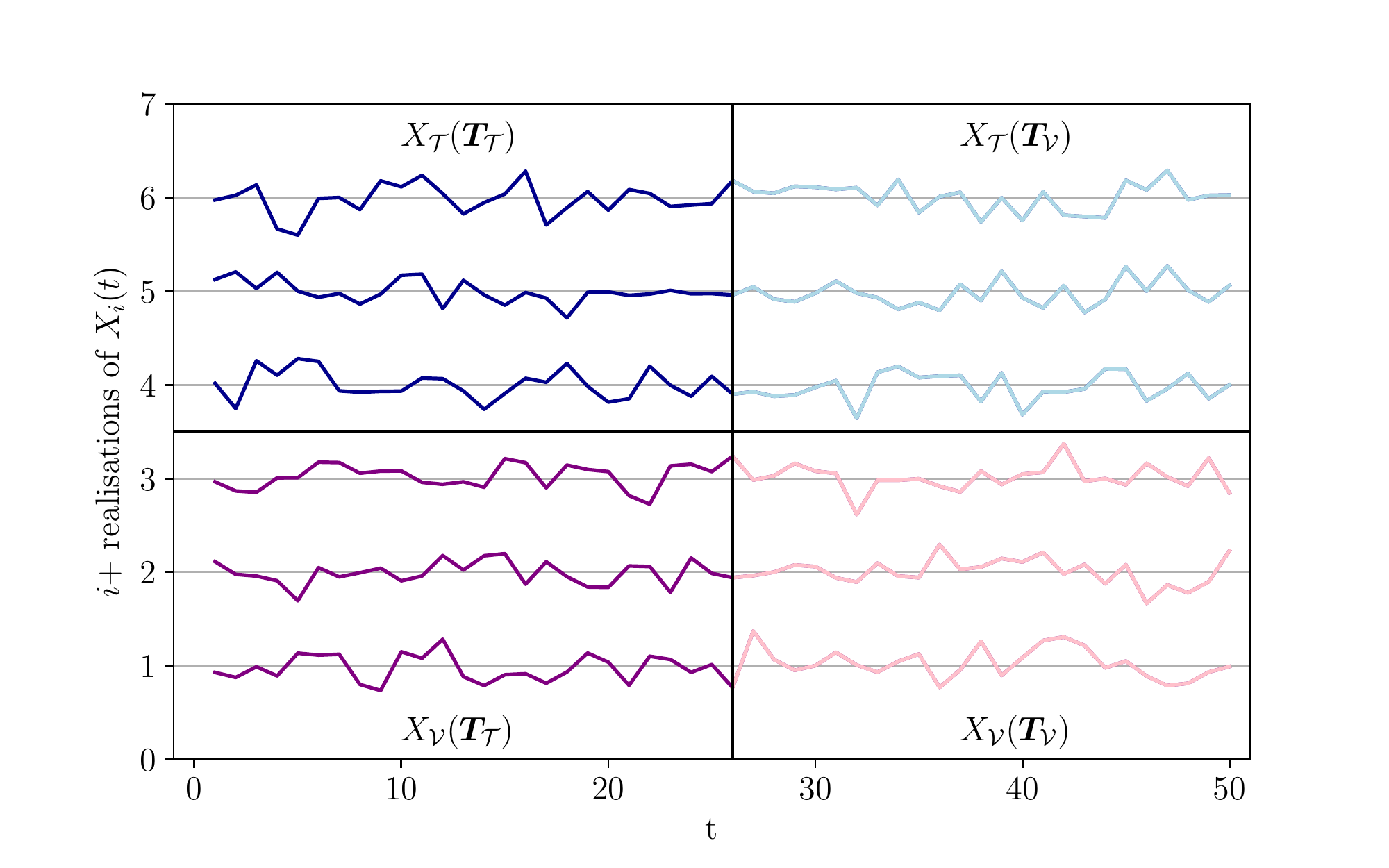}
	\vspace*{-0.25cm}
	\caption{Illustration of splitting multiple series over multiple timestamps into training and test sets.}
	\label{fig:stylised_multi_series_index_sets}
\end{figure}

Given a loss function, $L: \calX \times \calX \to \mathbb{R}$, a series $X_i$'s ground truth values at indices $t$, $X_i(t)$, and the corresponding predictions made by our strategy, $\widehat{X}_i(t)$, we aggregate the strategy's prediction errors over $N$ series and $K$ timestamps into a single estimate $\widehat{\varepsilon} = \frac{1}{NK} \sum_{i=1}^N \sum_{j=1}^K L \big( \widehat{X}_i(t_j), X_i(t_j) \big)$. This estimate can be shown to be an unbiased consistent estimator for the expected prediction error $\varepsilon  = \frac{1}{K} \sum_{j=1}^K \EE \big[ L \big( \widehat{X}_*(t_j), X_*(t_j) \big)  \big]$. Using the $K \times K$ unbiased sample covariance matrix of loss function evaluations, $\bSigma$, we also estimate the variance of $\widehat{\varepsilon} $ by $\widehat{v} = \frac{1}{N K^2} \  \sum_{i=1}^{K} \sum_{j=1}^{K} \bSigma_{ij}$. This can be shown to be an unbiased consistent estimator of $v = \Var[\widehat{\varepsilon}]$. In the outer cross-validation loop, we take the means of the single-fold estimates ($\widehat{\varepsilon}, \widehat{v}$)  across all 5 folds to produce the cross-validated estimates ($\widehat{\varepsilon}_{CV}, \widehat{v}_{CV}$). It can be argued that the cross-validated variance estimator $\widehat{v}_{CV}$ is conservative (i.e. overestimates the true variance) by assuming identical \& exchangeable single-fold errors $\widehat{\varepsilon}$ and applying \citet{bengio2004} Lemma 1.
 
Table \ref{tbl:experiments_combined_mean_plus_se} presents the empirical estimates of the generalisation error for both our series kernel GP-based prediction strategy and for some comparable strategies that are tuned \& estimated similarly. We report RMSE (Root Mean Square Errors) for our generalisation error estimates, having used a squared error function for the loss $L$ and applied the delta method to calculate the standard errors. Four simple baselines are also presented. We have conducted one-tailed two-sample t-tests to compare each prediction strategy's generalisation error to that of the best baseline for the given variable: $t = \frac{ \widehat{\varepsilon}_{CV}^{\text{ predictor}} - \widehat{\varepsilon}_{CV}^{\text{ best baseline}} }{ \sqrt{ \widehat{v}_{CV}^{\text{ predictor}} + \widehat{v}_{CV}^{\text{ best baseline}} } }$, and the degrees of freedom are $\sfrac{1}{5}$ the number of series per dataset.

It is worth noting that we were unable to evaluate the global alignment kernel on series from the Starlight dataset, hence the missing value in Table \ref{tbl:experiments_combined_mean_plus_se}. To perform these evaluations, we used the implementation of \citet{tslearn}, who based their code on algorithms by \citet{cuturi2011fast}. The time series in the Starlight dataset are the longest that we consider in this manuscript (1000 timestamps each), and evaluating the global alignment kernel on those series did not return valid results. In this manuscript, we do not investigate how to implement the global alignment kernel in a way that avoids such practical issues, since the algorithms involved are rather complicated. In contrast, and thanks to their simplicity, none of the series kernels we constructed encountered such difficulties in any of our experiments.

Returning to our main kernels of interest, we conclude that there is statistical evidence that our series kernel GPR-based prediction strategy outperforms the best baseline for all but one of the experiments. There is no such evidence for the majority of the simpler prediction strategies, including one designed specifically for forecasting (ARIMA). The notable exception to this is the global alignment kernel GPR-based prediction strategy, which also outperforms the best baseline on most datasets. Overall, we believe these results validate the use of our series kernel as part of a GPR-based prediction strategy.

\section{Discussion}
\label{section:discussion}

The main argument of our manuscript has been that it is useful to construct kernels that are designed to handle \textit{time series} data, rather than the more specific case of \textit{sequential} data, where the literature has until now focussed. Information is conveyed by the time indices assigned to observations in a time series, and this \textit{cardinal} information is richer than the \textit{ordinal} information conveyed by a simple sequential ordering. Other than our method of constructing series kernels, we are not aware of other kernel-based approaches that focus on time series, rather than simply sequences.

From a practical point of view, time series may be irregularly-spaced or multivariate, and we believe that it is important that kernels for time series be able to handle such cases. Although it is always possible to preprocess data, impute missing values, or ignore extraneous information that cannot be handled by a particular kernel, we argue that incorporating all available information from a time series can aid in making better predictions for real-life datasets.

One observation we would like to make is that there seems to be a lack of standardised irregularly-spaced datasets in the machine learning literature. In this manuscript, we made use of the irregularly-spaced Berkeley growth dataset collected by \citet{tuddenham1954physical} to validate our series kernel, and we hope that future large-scale benchmarking investigations can likewise incorporate similar datasets. It should be noted that the irregular spacing in the Berkeley growth dataset is rather mild and predictable, and it would be interesting to evaluate our series kernel on datasets where more values are missing, or where the irregular spacing is more severe. For example, in medical applications it is possible that sensor readings for different subjects can be taken at completely distinct timestamps, so that there are few or no common timestamps between time series at all; as we mentioned in Section~\ref{section:construction}, series kernels constructed according to our methodology easily handle such situations -- and should thus be considered for such applications. 

One difficulty of using highly irregularly-spaced data is that it would no longer be clear what benchmarks or alternative approaches one should be comparing to, as one would then diverge from the most common experimental settings in the literature. For this reason, most of the datasets we used in this manuscript have been regularly-spaced, and our aim has been to validate the usefulness of our series kernels in the most common setting.

The flexibility of our series kernels is thanks to the intuition of smoothing/interpolation that lies behind their construction. That is quite different to the intuition behind the global alignment kernel, which is based on scoring all possible alignments between sequences \cite{cuturi2007kernel,cuturi2011fast}. A second difference is that our methodology results in kernels designed for time series, while the global alignment kernel is designed specifically for sequences, which means the latter would ignore any information conveyed by time indices. A third difference is that our series kernels are relatively simple to implement, while the algorithms underpinning the evaluation of the global alignment kernel are rather complicated, and we encountered some numerical difficulties with the implementation we used in one experiment. On a related note, we have detailed a number of relatively simple ways to optimize evaluations of our series kernels in Section \ref{section:computation}. 

Besides alignment kernels, sequential kernels that are constructed from distances \cite{abanda2019review,pree2014general} are also able to represent sequential data. However, we are not aware of any such kernels that act on time series in their full generality, rather than on sequential data specifically. Many distances can handle multivariate data, though not all. Importantly, there is often no guarantee that kernels that are constructed from distance measures are even psd: for example, kernels constructed from the popular DTW distance are not psd without some further transformation \cite{abanda2019review}. By contrast, series kernels constructed according to our methodology are guaranteed the advantages that many distance-based kernels lack.

\section{Conclusion}
\label{section:conclusion}

We showed that kernels for time series can be easily constructed from well-known vector kernels. These present an alternative to the existing alignment-based and distance-based approaches to constructing kernels for sequential data, specifically. Furthermore, our series kernels natively possess the attractive properties of being psd, handling irregularly-spaced time series or those with missing values, and handling multivariate observations -- all without any additional preprocessing or regularization. We also discussed the intuition of smoothing/interpolation that lies behind these abilities. Empirically, we documented the utility of our series kernels for the traditional time series classification task, and also pursued a more interesting application of predicting series values using a Gaussian process-based regression strategy. 

We hope that future work will provide more extensive benchmarking-based evidence of various series kernels' performances on some common tasks, like \citet{bagnall16bakeoff} have done for time series classification. We also hope that future kernel-based methods will exploit our time series kernel methodology in new settings, especially for the problems of dimensionality reduction and multi-task learning. Finally, we hope that a kernel that is specifically designed for the real-life characteristics of time series (as opposed to sequences) may prove useful in practical applications.

\newpage

\section*{References}
\printbibliography[heading=none]

\end{document}